%% file: main.tex
\pdfoutput=1
\documentclass[11pt]{article} 

\usepackage{lmodern}
\usepackage{xspace}
\usepackage[protrusion=true,expansion=true]{microtype}

\usepackage[utf8]{inputenc} 
\usepackage[T1]{fontenc}   
\usepackage{url}            
\usepackage{booktabs}       
\usepackage{amsfonts}       
\usepackage{nicefrac}       
\usepackage{enumitem}
\usepackage{fullpage}
\usepackage{graphicx}
\usepackage{subfigure}
\usepackage{algorithmic}
\usepackage{algorithm}

\usepackage{amsmath}
\usepackage{amsthm}
\usepackage{amssymb}
\usepackage{hyperref}

\newcommand{\EE}{\mathbb{E}}
\newcommand{\RR}{\mathbb{R}}
\newcommand{\NA}{\textsc{Noise-Addition}}
\newcommand{\ignore}[1]{}

\newcommand{\ba}{a}
\newcommand{\bb}{b}
\newcommand{\bc}{c}
\newcommand{\bg}{g}
\newcommand{\bN}{N}
\newcommand{\bI}{I}
\newcommand{\bm}{m}
\newcommand{\bs}{s}

\newcommand{\bx}{x}

\newcommand{\bw}{w}
\newcommand{\btheta}{\theta}
\newcommand{\bDelta}{\Delta}
\newcommand{\citet}{\cite}

\newtheorem{Theorem}{Theorem}
\newtheorem*{Theorem*}{Theorem}
\newtheorem{Lemma}{Lemma}

\begin{document}

\title{AdaCliP: Adaptive Clipping for Private SGD}
\author{
\begin{tabular}[t]{c@{\extracolsep{6.5em}}c}
  Venkatadheeraj Pichapati & Ananda Theertha Suresh\\
 UC San Diego & Google Research\\ 
\small \texttt{dheerajpv7@ucsd.edu} & \small \texttt{theerta@google.edu}
\end{tabular}
\vspace{2ex}\\
\begin{tabular}[t]{c@{\extracolsep{6.5em}}c@{\extracolsep{6.5em}}c}
   Felix X. Yu & Sashank J. Reddi & Sanjiv Kumar\\
  Google Research & Google Research & Google Research \\
 \small \texttt{felixyu@google.com} & \small\texttt{sashank@google.com}
 & \small \texttt{sanjivk@google.com}
\end{tabular}
\vspace{2ex}\\  
}
\maketitle

\input{abstract.tex}
\input{introduction.tex}
\input{motivation.tex}
\input{noise.tex}

\input{experiments.tex}
\input{conclusion.tex}
\newpage
\bibliography{ref}
\bibliographystyle{unsrt}
\newpage
\appendix

\section*{\centering Appendix - AdaCliP: Adaptive Clipping for Private SGD}
\input{convergence.tex}
\input{discussions.tex}
\end{document}

%% file: abstract.tex
\begin{abstract}
  Privacy preserving machine learning algorithms are crucial for learning models over user data to protect sensitive information. Motivated by this, differentially private stochastic gradient descent (SGD) algorithms for training machine learning models have been proposed. At each step, these algorithms modify the gradients and add noise proportional to the sensitivity of the modified gradients. Under this framework, we propose AdaCliP, a theoretically-motivated differentially private SGD algorithm that provably adds less noise compared to the previous methods, by using coordinate-wise adaptive clipping of the gradient. We empirically demonstrate that AdaCliP reduces the amount of added noise and produces models with better accuracy. 
\end{abstract}

%% file: introduction.tex
\newcommand{\theertha}[1]{{\color{red} th: #1}}

\section{Introduction}
Machine learning models are widely deployed in various applications such as image classification~\cite{KaimingXSJ15,AlexIG12}, natural language processing~\cite{TomasMLJS10,OriolLTSIG15}, and recommendation systems~\cite{Jordan08}. Most state-of-the-art machine learning models are trained on user data, examples include keyboard models \cite{hard2018federated}, 
automatic video transcription \cite{kumar2017lattice} among others. User data often contains sensitive information such as typing 
histories, social network data, financial and medical records. Hence releasing such machine learning models to public requires rigorous privacy guarantees while maintaining the performance.

Of the various privacy mechanisms, differential privacy~\cite{CynthiaFKA06} has emerged as the well accepted notion of privacy. The notion of differential privacy provides a strong notion of individual privacy while permitting useful data analysis in machine learning tasks.
We refer the reader to~\cite{CynthiaA14} for a survey. Originally used for database queries, it has been adapted to provide privacy guarantees for machine learning models.  Informally, for the output to be differentially private, the estimated model and all of its parameters should be indistinguishable whether a particular client’s data was taken into consideration or not.

Differential privacy for machine learning has been studied in various models including models with convex objectives \cite{KamalikaC09,KamalikaCA11,XiFAKSJ17} and more recently deep learning methods~\cite{RezaV15,MartinAIBIKL16,BrendanDKL17}.
One particular set of algorithms for learning differentially private machine learning models can be interpreted as \emph{noisy stochastic gradient descent} (SGD)~\cite{RaefAA14, RezaV15, MartinAIBIKL16}. At each iteration of SGD, these algorithms modify the gradients suitably to provide differential privacy.

In this paper, we ask if there is a systematic, theoretically motivated, principled approach to obtain an optimal modification strategy. Motivated by the convergence guarantees of SGD, we propose a new differentially private SGD algorithm called AdaCliP. Compared to the previous methods, AdaCliP achieves the same  privacy guarantee with much less added noise by using coordinate-wise adaptive clipping of the gradient. Since the convergence of SGD depends on the variance of the gradient, this approach improves the learned model quality. 
We empirically evaluate the performance of differentially private SGD techniques on MNIST dataset using various machine learning models, including neural networks. Our experiments show that AdaCliP achieves much better accuracy than previous methods for the same privacy constraints. We also empirically evaluate performance of momentum optimization algorithm in place of SGD and show that momentum does not result in models with better accuracy even though it adds less noise per iteration compared to SGD. We provide a possible explanation for this counter-intuitive phenomenon in Appendix~\ref{sec:momentum}.

The paper is organized as follows. In Section~\ref{sec:previous_work}, we overview differential privacy and previous methods. In Section~\ref{sec:motivation}, we motivate the need for a new differentially private SGD method. 
In Section~\ref{sec:general}, we introduce a general formulation that encompasses previous methods and present Theorem~\ref{thm:main} to show the parameters that minimize the amount of noise added. In Section~\ref{sec:adaclip}, we state our SGD technique AdaCliP that uses optimal parameters derived in Theorem~\ref{thm:main}.  In Section~\ref{sec:experiments}, we present our empirical results. 

%% file: motivation.tex
{\bf Notation.} For any vectors $u$ and $v$, $u/v$ and $u v$ are used to denote element-wise division and element-wise multiplication respectively. For any vector $v \in R^d$, $v_i$ is used to denote the $i^{\text{th}}$ co-ordinate of the vector. For a vector $u$, $\lVert u \rVert$ denotes the $\ell_2$ norm of vector $u$.

\section{Differential privacy for distributed SGD}
\label{sec:previous_work}
We first formally describe differential privacy~\cite{CynthiaKFIM06} and the previous differentially private SGD methods.
We then motivate the need for a new differentially private SGD algorithm by a simple example.

\subsection{Differential privacy}
\label{ssec:dp}
Let $\mathcal{D}$ be a collection of datasets. Two datasets $D$ and $D'$ are adjacent if they differ in at most one user data. A mechanism
$\mathcal{M}: \mathcal{D} \rightarrow \mathcal{R}$ with domain $\mathcal{D}$
and range $\mathcal{R}$ is $(\epsilon,\delta)$-differentially private if for 
any two adjacent datasets $D, D' \in \mathcal{D}$ and for any subset of outputs $S \subseteq \mathcal{R}$,
\[
\text{Pr} [\mathcal{M}(D) \in S ] \le e^{\epsilon} \text{Pr} [\mathcal{M}(D') \in S] + \delta.
\]
If $\delta = 0$, it is referred to as pure differential privacy~\cite{CynthiaA14}. The $(\epsilon,\delta)$- formulation 
allows for pure formulation to break down with probability $\delta$. We state our results with $(\epsilon,\delta)$- privacy formulation, but it can be easily extended to $(\epsilon,0)$ pure differential privacy. A standard paradigm to provide privacy-preserving approximations to  function $\varphi: \mathcal{D} \rightarrow \mathcal{R}^d$ is to add noise 
proportional to the sensitivity $S_\varphi$ of function $\varphi$, which is formally defined as the maximum of absolute $\ell_2$ difference between function values for two adjacent 
datasets $D$ and $D'$ i.e., 
\[
S_\varphi = \max_{D, D' \in \mathcal{D}} \lVert\varphi(D) -  \varphi(D')\rVert.
\]
One such privacy-preserving approximation is the Gaussian mechanism~\cite{CynthiaA14} that adds Gaussian noise of variance of 
$S_\varphi^2 \sigma^2$, i.e., 
$
\mathcal{M}(D) = \varphi(D) + \mathcal{N}(0, S_\varphi^2\sigma^2 \bI),
$
where $\mathcal{N}(\mu, \Sigma)$ represents Gaussian variable with mean $\mu$ and covariance matrix $\Sigma$. We now present a well-known result~\cite{CynthiaA14} that relates noise scale $\sigma$ of Gaussian  mechanism to parameters $\epsilon$ and $\delta$. 
\begin{Lemma}
For any $\epsilon < 1$, the Gaussian mechanism with noise scale $\sigma$ satisfies $\left(\epsilon, \frac45\exp\left(\frac{-(\sigma\epsilon)^2}{2}\right)\right)$-differential privacy.
\label{lemma:gaussian}
\end{Lemma}
 Lemma~\ref{lemma:gaussian} implies various $(\epsilon, \delta)$ pairs for a given noise scale
$\sigma$.
The above definition uses $\ell_2$ sensitivity. For databases with $\ell_1$ sensitivity, Laplace noise can be added to obtain pure differential privacy~\cite{CynthiaA14}. Recently the optimal noise distributions to generate the least amount of noise have been proposed for both $(\epsilon, 0)$ and $(\epsilon, \delta)$ differential privacy~\cite{geng2016optimal, geng2018optimal}.
\vspace{-0.5em}
\subsection{Differential privacy for machine learning}
\vspace{-0.5em}
Differential privacy definition was originally used to provide strong privacy guarantees for database querying and since used in several applications~\cite{BoazKCSFK07,KamalikaN06,FrankK07}. Recently it has been extended to machine learning formulations. For the context of machine learning, dataset $D$ is a collection of user data and the function $\mathcal{M}$ corresponds to the output machine learning model parameters. We note that this notion of differential privacy is also called \emph{global differential privacy}.

Differential privacy for machine learning models can be obtained in four ways: \emph{input perturbation}, \emph{output perturbation}, \emph{objective perturbation}, and \emph{change in optimization algorithm.}

In input perturbation, the dataset $D$ is first modified using Laplace or Gaussian mechanism and the resulting perturbed dataset is used to train the machine learning model~\cite{JohnMM13}. In output perturbation techniques, the machine learned model is trained completely and then the final model is appropriately changed by using exponential mechanism~\cite{FrankK07,KamalikaCA11,KamalikaS13,RaefAA14} or by adding Laplace or Gaussian noise to the final model~\cite{BenjaminPLN09,KamalikaC09,XiFAKSJ17}. In objective perturbation techniques, the objective function is perturbed by the appropriate scaling of Laplace or Gaussian noise and the machine learning model is trained over perturbed objective function~\cite{KamalikaC09,KamalikaCA11,JunZXYM12}.

The fourth method modifies the optimization algorithm for training machine learning models. This includes noisy SGD methods, which we discuss in the next section.
\vspace{-0.5em}
\subsection{Noisy SGD methods}
\vspace{-0.5em}
SGD and its variations such as momentum~\cite{Boris64}, Adagrad~\cite{JohnEY11}, or Adam~\cite{DiederikJ14} are used for training machine learning models. These algorithms can be modified by adding noise to their gradients at each iteration to provide differentially private machine learning algorithms. Even though noisy SGD usually provides global differential privacy, recent works have shown that they can be combined with cryptographic homomorphic encryption techniques to provide stronger privacy guarantees~\cite{bonawitz2017practical, AgarwalSureshYuKumarMcMahan2018}.

\begin{figure}
    \centering
    \includegraphics[width=0.48\textwidth]{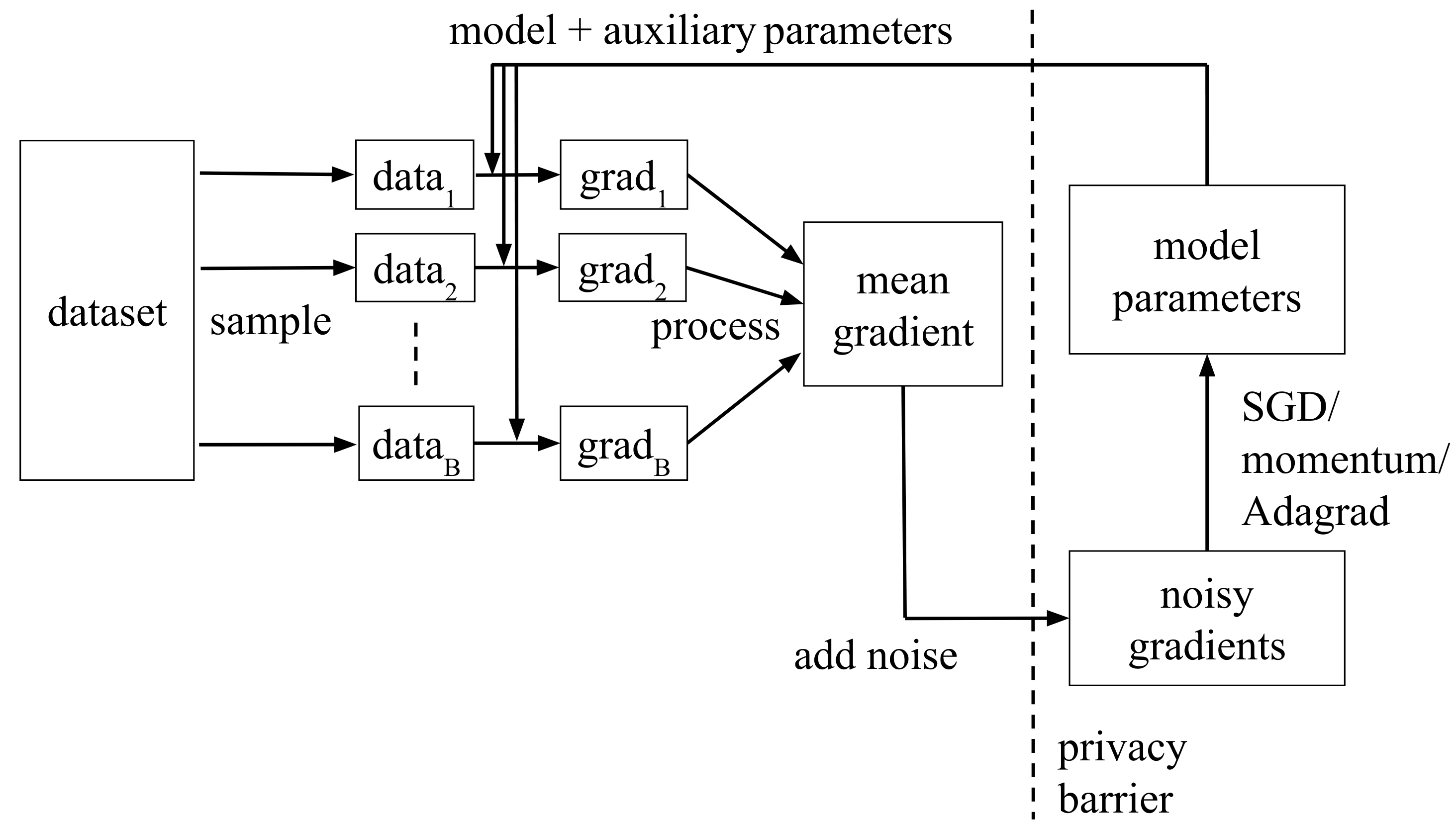}
    \caption{Outline of differentially private SGD algorithms. Parameters to the right of the privacy barrier are differentially private.}
    \label{fig:noisy_sgd}
\end{figure}
Differentially private SGD algorithms is outlined in Figure~\ref{fig:noisy_sgd}. At each round of SGD, the algorithm selects a subset of data. Using the current model and auxiliary parameters, it computes gradients on each data point, and optionally modifies (e.g. clipping) the gradients. It then computes the mean of the gradients, adds noise to the mean, and uses the noisy gradient to update the model. The analysis of such algorithms can be broken into two parts:
\begin{itemize}[leftmargin=0.7cm]
    \vspace{-0.5em}
    \item Obtain $(\epsilon',\delta')$ differential privacy for each round of SGD, by ensuring that any information from the dataset that 
    is used to update the model parameters is differentially-private.
    \vspace{-0.5em}
    \item Compute the total privacy cost of all SGD iterations to obtain overall $(\epsilon, \delta)$ parameters.
\end{itemize}

We first consider the second part.  Suppose we show that the noisy gradients sent from the dataset to the server is $(\epsilon', \delta')$- differentially private.  To keep track of accumulated privacy loss over multiple iterations of SGD, a privacy accountant is used~\cite{Frank09}. Privacy accountant maintains accumulated privacy loss in terms of $\epsilon$ and $\delta$, which are determined by the composition theorem used, $\epsilon', \delta'$ used in each iteration.
For the Gaussian mechanism, \citet{MartinAIBIKL16} introduced moments accountant, which provides tighter privacy bounds compared to other composition theorems~\cite{CynthiaGS10,CynthiaJ09,MarkT16,CynthiaJ09,PeterSP17}. Recently, ~\citet{yu2019differentially, wu2019p3sgd}  proposed adaptive strategies
to select privacy parameters $\epsilon'$, $\delta'$ for each iteration. The differentially private SGD algorithm terminates the training once the privacy budget is reached. 

For the first part, recall that  a common technique to provide privacy-preserving 
approximation is to bound the sensitivity of the function and add Gaussian 
noise proportional to the sensitivity bound. To this end, we need to bound the sensitivity of the gradients at each round of SGD. This can be achieved in several ways.

If the loss function is differentiable (if not differentiable use sub-gradients) and Lipschitz bounded, \citet{RaefAA14} bounds the gradient norm by the Lipschitz bound and use it to derive the sensitivity of gradients. If the loss function derivative is bounded as a function of input (for example, in the
logistic regression case, one can bound the gradient norm by the maximum input norm possible) and hence derive the sensitivity of gradients. If the loss function does not have known Lipschitz bound as in deep learning applications, apriori bounds on gradient norm are difficult to derive. At each iteration of training,~\citet{ZhangST18} proposes to use public data to obtain an approximate bound on gradient norm  and clip the gradients at this approximate bound. However the availability of public data is a strong assumption and~\citet{RezaV15,MartinAIBIKL16} clip the gradients without the availability of public data. We also assume no access to public data. 

At each iteration of training,~\citet{RezaV15} clips each coordinate of the stochastic gradient vector to the range $[-C,C]$. \citet{MartinAIBIKL16} bounds the $\ell_2$-norm of the stochastic gradient by clipping the gradient $\ell_2$-norm to a threshold $C$, 
where if the gradient $\ell_2$-norm is more than $C$, each gradient entry is scaled down by a 
factor of $C$ divided by the $\ell_2$-norm of the gradient. This ensures that the $\ell_2$
sensitivity is bounded by $C$. Then the clipped gradients are  averaged over the batch and noise $\mathcal{N}(0, \sigma^2C^2 \bI)$ is added to the average of the clipped gradients. The noisy clipped gradient mean is used to update the model during this iteration and the noise scale $\sigma$ determines the privacy cost of this iteration. Since~\citet{RezaV15} clips each gradient entry and~\citet{MartinAIBIKL16} 
clips the entire gradient norm, for same clip thresholds,~\citet{RezaV15} incurs much more privacy loss compared to~\cite{MartinAIBIKL16}. Recently, \citet{van2018three,  thakkar2019differentially} proposed adaptive strategies to select the $\ell_2$-norm threshold $C$. In contrast, our algorithm adaptively selects coordinate-wise clip thresholds.
\vspace{-0.5em}
\section{Motivation for AdaCliP}
\label{sec:motivation}
\vspace{-0.5em}
The above set of noisy SGD methods raises several questions: is $\ell_2$ clipping provably the best clipping strategy? If so, how do we choose the clipping threshold $C$? Is there a systematic, theoretically motivated, principled approach to obtain an optimal clipping strategy? For example, instead of adding noise, if we whiten the gradients by dividing by the standard deviation and add noise, is it better? We answer these questions, by deriving a theoretically motivated, principled clipping strategy, which provably adds less noise compared to the previous methods. Before we proceed further, we motivate AdaCliP by analyzing previous methods~\cite{RaefAA14,RezaV15,MartinAIBIKL16} on a simple regression problem:
\begin{equation}
\label{eq:regression}
\mathop{\arg\min}_{\btheta \in \RR^d} \frac1{2n} \sum_{i=1}^n \lVert\btheta-\bx^i\rVert^2,
\end{equation}
where $\bx^i \in \mathrm{R}^d$. The solution for this $\ell_2^2$-regression problem is 
 $\btheta^* = \frac1n \sum_{i=1}^n \bx^i$. Let $\bx^i = (y_i \mu, 0,..,0) $, where $n/2$ values of $y_i$ are $-1$ and $n/2$ values are $+1$. 

We now analyze the performance of differentially private SGD algorithms. At iteration $t$, the gradient with respect to any example $\bx^{i_t}$ is $\bg^t = \btheta^t -\bx^{i_t}$. Notice that revealing the gradient $\bg^t$ and $\bx
^{i_t}$, reveals the same information. Further observe that 
\[
\EE \lVert\btheta^t -\bx^{i_t}\rVert^2 = \EE \lVert\btheta^t \rVert^2 + \EE \lVert\bx^{i_t}\rVert^2 \geq \EE\lVert\bx^{i_t}\rVert^2 = \mu^2,
\]
where the first equality follows by observing that $\EE \bx^{i_t} = 0$.
Since noise added is proportional to the $\ell_2$ norm of the vector revealed, it is beneficial to reveal $\bx^{i_t}$. Consider the clip threshold $C = \mu$. Hence noise added is $\mathcal{N}(0, \sigma^2\mu^2 \textbf{I})$, where $\sigma$ is computed using privacy parameters $\epsilon, \delta$, and the number of rounds. Therefore, $\ell_2^2$-norm of added noise is $\sigma^2\mu^2 d$.
Hence the signal to noise ratio (SNR) (ratio of $\ell_2^2$ norm of clipped gradient to that of noise added) is $\frac1{\sigma^2 d}$. 

In this example, SNR gets worse with the size $d$ of dimensions even though only
one of the dimensions contains information. Based on the definition of differential privacy, one need not add much noise to dimensions other than the first one. 
This motivates us to adaptively add different noise levels to different dimensions to minimize the $\ell_2^2$-norm of the added noise.  We note that the above analysis can be easily modified to other techniques such as the Lipschitz bounded sensitivity \cite{RaefAA14}. Further, it is easy to check that SNR stays the same in the above analysis for any clip threshold less then $\mu$ and gets only worse if clip threshold is greater than $\mu$.

%% file: noise.tex
\section{Theoretical analysis}
\label{sec:general}
Before we present AdaCliP, we first state a general convergence result for SGD for non-convex functions.
For a statistic $\hat{\alpha}$ that serves as an estimate of parameter $\alpha$, the bias of $\hat{\alpha}$ is defined as $\text{bias}(\hat{\alpha}) \stackrel{\Delta}{=} \lVert\EE \hat{\alpha} - \alpha\rVert$ and the variance of $\hat{\alpha}$ is defined as 
$\text{Var}(\hat{\alpha}) \stackrel{\Delta}{=} \EE\lVert\hat{\alpha}-\EE\hat{\alpha}\rVert^2$.
\citet[Theorem 1]{reddi2016stochastic} can be modified to show the following lemma.
\begin{Lemma}
\label{lem:sgd_convergence}
Let $f(\btheta) = \frac{1}{N}\sum^N_{k=1} f_k(\btheta)$.  For a suitable choice of learning rate, iterates of SGD satisfy
\[
\frac{1}{T} \sum^T_{t=1} \EE \lVert \nabla  f(\btheta^t) \rVert^2
\leq 
\frac{c}{T} \sqrt{\sum^T_{t=1}\text{Var}(\bg^t)} + c \max_{1 \leq t \leq T} \text{bias}(\bg^t),
\]
where $\bg^t$ is the stochastic gradient at time $t$ and $c$ is a constant.
\end{Lemma}
Previous algorithms added noise to the gradients themselves. In a more general framework, one can transform the gradient by a function, add noise, and apply the inverse of the function back. This may reduce the variance and bias of differentially private gradients and by Lemma~\ref{lem:sgd_convergence} yield a better solution. We consider the  class of element-wise linear transformations and find the best transformation.
\subsection{General framework}
 Let $\bg^t = (g_1^t,g_2^t,..,g_d^t)$ be the stochastic gradient vector at iteration $t$. 
Let $\ba^t = (a_1^t,a_2^t,..,a_d^t)$ and $\bb^t = (b_1^t, b_2^t,..,b_d^t)$  be the auxiliary vectors that will be described later.
Transform $\bg^t$ by subtracting $\ba^t$ from it and dividing each
dimension of $\bg^t-\ba^t$ by that of $\bb^t$.
Let $\bw^t = \frac{\bg^t-\ba^t}{\bb^t}$ be the transformed gradient i.e., $
w_i^t = \frac{g_i^t-a_i^t}{b_i^t}.
$
To bound the sensitivity, the transformed gradient is clipped at norm 1. Let 
the clipped transformed gradient be $\hat{\bw}^t$.
\vspace{-2ex}
\[
\hat{\bw}^t = \text{clip}(\bw^t, 1) \stackrel{\Delta}{=} \frac{\bw^t}{\max(1,\lVert\bw^t\rVert_2)}.
\]
We add noise $\mathcal{N}(0,\sigma^2\bI)$ to the clipped transformed 
gradient $\hat{\bw}^t$. Let the noisy gradient be 
$\tilde{\bw}^t$.
\[
\tilde{\bw}^t = 
\hat{\bw}^t + \bN^t \quad \quad \quad \bN^t \sim \mathcal{N}(0, \sigma^2 \bI ),
\]
where $\sigma$ is determined by the privacy parameters.
Rescale $\tilde{\bw}^t$ to the same scale as original gradient by 
multiplying each dimension of it with that of $\bb^t$ and adding 
$\ba^t$ to resulting vector. 
\[
\tilde{g}^t = b^t\tilde{w}^t + a^t.
\]
Finally, output $\tilde{\bg}^t$ as privacy-preserving approximation of $\bg^t$.

Note that choices of $\ba^t = (0,0,...,0)$ and 
$\bb^t = (C,C,...C)$ result in the algorithm of~\cite{MartinAIBIKL16}. A natural question is to ask is: what are the optimal choices of $\ba^t$ and $\bb^t$? By Lemma~\ref{lem:sgd_convergence}, observe that we are interested in the variance and bias of the new gradient $\tilde{\bg}^t$. By triangle inequality and Jensen's inequality,
\[
\text{bias}(\tilde{\bg}^t) \leq \text{bias}({\bg}^t) + 2 \EE \lVert \tilde{\bg}^t - \bg^t \rVert
\text{ and }
\text{Var}(\tilde{\bg}^t) \leq 3\text{Var}({\bg}^t) + 6 \EE \lVert \tilde{\bg}^t - \bg^t \rVert^2.
\]
Hence, to find optimal values of $\ba^t$ and $\bb^t$ we would like to bound $\EE \lVert \tilde{\bg}^t - \bg^t \rVert^2$. A straightforward calculation shows that the above quantity can be simplified to
\[
 \EE \lVert \tilde{\bg}^t - \bg^t \rVert^2 
 =\lVert  \bg^t -  \ba^t\rVert^2 \left(1 - \frac{1}{\max(1, \lVert  \bw^t \rVert)} \right)^2 +  \lVert\bb^t \rVert^2 \sigma^2.
\]
Thus there are two potential sources for gradient modification. The first term in the above equation corresponds to the case when the transformed gradient $\bw^t$ might get clipped. The second term corresponds to the Gaussian noise injected to the clipped gradient. Ideally, we would like to find the best $\ba^t$ and $\bb^t$ that minimize the above expression. However, it is difficult to analyze the effect of clipping on the convergence. Hence we try to limit clipping, by assuming that 
\[
\EE \lVert\bw^t\rVert^2 \le \gamma
\]
in analysis and try to minimize the injected Gaussian noise. Observe that $\EE \lVert\bw^t\rVert^2  \le \gamma$ ensures that $\lVert\bw^t\rVert^2 > 1$ with constant probability (by Markov's inequality) and hence $\bw^t$ gets clipped 
with constant probability. Later in Theorem~\ref{thm:adaclip_convergence}, we analyze the convergence of the proposed method by using the above stated concentration bound. Therefore we limit the 
choices of $\ba^t$ and $\bb^t$ such that $\EE \lVert\bw^t\rVert^2  \leq \gamma$ and find the optimal $\ba^t$ and $\bb^t$ that minimize the Gaussian noise. Interestingly, we show that optimal $\ba^t$ and $\bb^t$ is different than the traditional whitening choice.
In Section~\ref{sec:adaclip}, we propose methods to approximate optimal $a^t$ and $b^t$ using differentially private gradients.

\begin{Theorem}
\label{thm:main}
If $\EE \lVert\bw^t\rVert^2 \leq \gamma$, the
expected $\ell_2$-norm of added noise i.e., $\lVert \bb^t \rVert \sigma$ is minimized when
\setlength\abovedisplayskip{0pt}
\[
a_i^t = m_i^t \stackrel{\Delta}{=} \EE g_i^t
\text{ and }
b_i^t = \sqrt{\frac{s_i^t}{\gamma}} \cdot \sqrt{ \sum_{i=1}^d s_i^t}, 
\]
where $s_i^t \stackrel{\Delta}{=} \sqrt{\EE (g_i^t - \EE g_i^t)^2}$.
Expected $\ell_2^2$-norm of added noise is $\propto (\sum_{i=1}^d s_i^t)^2/\gamma$ .
    \end{Theorem}
\begin{proof} 
Observe that  $\EE \lVert\bw^t\rVert^2 $ can be rewritten as 
\begin{equation}
\EE \lVert\bw^t\rVert^2  = \sum_{i=1}^d \frac{(s_i^t)^2 + (m_i^t - a_i^t)^2}{(b_i^t)^2},
\label{eq:norm_bounding}
\end{equation}
From Equation~\ref{eq:norm_bounding},
the condition $\EE\lVert\bw^t\rVert^2 \le \gamma$ implies that 
$\sum_{i=1}^d \frac{(s_i^t)^2 + (m_i^t - a_i^t)^2}{(b_i^t)^2} \le \gamma$.

Further recall that added Gaussian vector is $b^t \cdot N$ and it's expected $\ell_2^2$-norm is $\sigma^2 \sum_{i=1}^d (b_i^t)^2$.

Hence to minimize expected $\ell_2^2$-norm of added noise, one should minimize 
$\sum_{i=1}^d (b_i^t)^2$. 

Hence it results in the optimization problem,
\begin{align*}
\min_{\ba^t, \bb^t} \sum_{i=1}^d (b_i^t)^2 \qquad
\text{s.t. }&\sum_{i=1}^d \frac{(s_i^t)^2 + (m_i^t - a_i^t)^2}{(b_i^t)^2} \le \gamma.
\end{align*}
Observe that by Holder's inequality,
\begin{equation}
    \label{eq:optimal}
\left(\sum_{i=1}^d (b_i^t)^2\right)\left(\sum_{i=1}^d \frac{(s_i^t)^2}{(b_i^t)^2}\right) \ge \left(\sum_{i=1}^d s_i^t\right)^2
\end{equation}
and hence 
\begin{align*}
 \left(\sum_{i=1}^d (b_i^t)^2\right) &\ge  \left(\sum_{i=1}^d s_i^t\right)^2 \bigg / \left(\sum_{i=1}^d \frac{(s_i^t)^2}{(b_i^t)^2}\right)
\ge \left(\sum_{i=1}^d s_i^t\right)^2 \bigg / \gamma .
\end{align*}
where last equation follows from constraint $\sum_{i=1}^d \frac{(s_i^t)^2 + (m_i^t - a_i^t)^2}{(b_i^t)^2} \le \gamma$. The last inequality is satisfied with equality when $b_i^t = \sqrt{s_i^t /\gamma} \cdot \sqrt{ \sum_{i=1}^d s_i^t}$. Further, combined with choice 
of $a_i^t = m_i^t$, 
\[
\sum_{i=1}^d \frac{(s_i^t)^2 + (m_i^t - a_i^t)^2}{(b_i^t)^2} = \sum_{i=1}^d \frac{\gamma (s_i^t)^2}{s_i^t \cdot \sum_{i=1}^d s_i^t} = \gamma,
\]
satisfying the constraint. Hence, the expected $\ell_2^2$-norm of added noise is 
\[
\sigma^2 \sum_{i=1}^d (b_i^t)^2 = \sigma^2 \sum_{i=1}^d s_i^t (\sum_{i=1}^d s_i^t)/\gamma = \sigma^2 \left(\sum_{i=1}^d s_i^t\right)^2 /\gamma.\qedhere
\]
\end{proof} 
\vspace{-1em}
Note that $a_i^t = m_i^t$ and 
$b_i^t = \sqrt{d} s_i^t /\sqrt{\gamma}$ leads to traditional whitening of gradient and ensures that
$\EE \lVert\bw^t\rVert^2 = \gamma$. Interestingly, the optimal choice for $\bb^t$ is different from the traditional
whitening choice. 
The classic whitening results in
added noise with expected $\ell_2^2$-norm of 
\[
\sigma^2 \sum_{i=1}^d (b_i^t)^2 = \sigma^2 \sum_{i=1}^d d (s_i^t)^2 /\gamma = \sigma^2 d \sum_{i=1}^d (s_i^t)^2/\gamma.
\]
By the Cauchy-Schwarz inequality,
$
d \sum_{i=1}^d (s_i^t)^2 \ge \left(\sum_{i=1}^d s_i^t\right)^2,
$
equality only when all $s_i^t$ are equal. 
Hence, the proposed approach adds less noise compared to $\ell_2$ clipping and whitened gradients in most cases. 
\subsection{Convergence analysis}
In this section, we present the convergence analysis for AdaCliP. 
Our main result is the convergence of this algorithm for general nonconvex functions, the proof of which is provided in the Appendix~\ref{app:convergence}.

\begin{Theorem}
Suppose the function $f$ is L-Lipschitz smooth, $\|\nabla f_k(\btheta)\| \leq G$ for all $\theta \in R^d$ and $k \in [N]$, and $\EE_k\|\nabla f_k(\btheta) - \nabla f(\theta)\|^2 \leq \sigma_g^2$ for all $\theta \in R^d$. Furthermore, suppose learning rate $\eta^t = \eta < \tfrac{1}{3L}$. Then, for the iterates of AdaCliP 
with batch size 1 and $a_t = \EE[g_t]$, 
\begin{align*}
\frac{1}{T}\sum_{t=0}^{T-1} & \EE[ \|\nabla f(\btheta^t)\|^2] \leq \frac{2[f(\btheta^0) - f(\btheta^*)]}{\eta T} + \underbrace{3L\eta\sigma_g^2}_{\substack{\text{Stochastic gradient} \\ \text{variance}}} 
 + \underbrace{\frac{6G^2}{T} \sum_{t=0}^{T-1} \EE\left\|\frac{\bs^t}{\bb^t}\right\|^2}_{\text{Clipping bias}} 
 + \underbrace{\frac{L\eta\sigma^2}{T} \sum_{t=0}^{T-1} \|\bb^t\|^2}_{\text{Noise-addition variance}},
\end{align*}
where $s_i^t \stackrel{\Delta}{=} \sqrt{\EE (g_i^t - \EE g_i^t)^2}$.
\label{thm:adaclip_convergence}
\end{Theorem}

Note the dependence of convergence result on the variance of stochastic gradients, bias introduced due to clipping and variance due to noise addition. The terms of special interest to us are: clipping bias and noise-addition variance. There is an inherent trade-off between these two terms as observed through the dependence on $b^t$. One can decrease the clipping bias by increasing $\|b^t\|$ but this comes at the expense of larger noise addition. One can optimize the values of $b^t$ to minimize this upper bound. In doing so, our choice of $b^t$ in Theorem~\ref{thm:main} again becomes quite evident. In particular, observe that clipping bias and noise-addition variance are the two terms in the LHS of Eq.~\eqref{eq:optimal}. Thus, by Holder's inequality, their weighted sum is minimized when $\bb^t$ is chosen as per Theorem~\ref{thm:main}.
In the following section, we discuss choices of $\ba^t$ and $\bb^t$ and their convergence bounds. In specific we show how they affect the last term in Theorem \ref{thm:adaclip_convergence}.
\subsection{Comparison on regression}
We now revisit the regression 
problem shown in (\ref{eq:regression}). Recall that in this example, all gradients have the same norm and hence we can set clipping threshold to $\mu$. Hence, the clipping bias is $0$. With this choice of the clipping threshold, we compare various choices of $\ba^t$ and $\bb^t$.
\begin{itemize}[leftmargin=0.5cm]
    \item \cite{MartinAIBIKL16} is equivalent to using 
choices $a_i^t = 0$ and $b_i^t = \mu$. Hence, $\ell_2^2$-norm of added noise is 
$\sigma^2 \sum_{i=1}^d \mu^2 = d \sigma^2 \mu^2$. 
\item If we whiten the gradients,
i.e., $\ba^t = \bm^t = \EE \btheta^t$ and 
$b_i^t = \sqrt{d} s_i^t = \sqrt{d}\sqrt{\EE(g_i^t - \EE g_i^t)^2}$.
Specifically, $a_i^t = 0$ $\forall i$, $b_1^t = \sqrt{d}\mu$ and $b_i^t = 0$ $\forall i > 1$. \footnote{Notice that to avoid definitions of 0/0, one can consider arbitrarily small variances in dimensions 2 to d.  Our observations hold even in that case.} Hence, $\ell_2^2$-norm of added noise is $\sigma^2 d \mu^2$, same as that of~\citet{MartinAIBIKL16}.
\item For the optimal choices i.e., $\ba^t = \bm^t$ 
and $b_i^t = \sqrt{s_i^t} \cdot \sqrt{\sum_{i=1}^d s_i^t}$. Specifically, 
$a_i^t = 0$ $\forall i$, $b_1^t = \mu$ and $b_i^t = 0$ $\forall i >1$. Hence, 
$\ell_2^2$-norm of added noise is $\sigma^2 \mu^2$, a factor of $d$ less than that of~\citet{MartinAIBIKL16} and whitening.
\end{itemize}
\section{AdaCliP}
\label{sec:adaclip}
\begin{algorithm}[ht]
\caption{AdaCliP}
\begin{algorithmic}[1]
\STATE \textbf{Inputs:} objective function $f(\theta) = \frac{1}{N}
\sum_{k=1}^N f_k(\theta)$, learning rate $\eta^t$,  batch size $B$,
noise scale $\sigma$.
\STATE Initialize $\bm^0 = 0 \cdot 1$, $\bs^0 = \sqrt{h_1 h_2} \cdot 1$
\STATE Initialize $\theta^0$ randomly
\FOR{$t=0$ to $T-1$}
\FOR{$i=1$ to $d$}
\STATE $b_i^t$ = $\sqrt{s_i^t} \cdot \sqrt{\sum_{i=1}^d s_i^t}$
\ENDFOR
\STATE $S_t \leftarrow$ $B$ random users 
\FOR{each $k \in S_t$} 
\STATE Compute gradient: $\bg^t(k) = \nabla f_k(\theta^t) $
\STATE Privacy preserving noise addition: $\tilde{\bg}^t(k) = \NA(\bg^t(k), \bm^t, \bb^t, \sigma)$
\ENDFOR
\STATE Compute average noisy gradient:
$\tilde{\bg}^t = \frac1{B}\sum_{k \in S_t} \tilde{\bg}^t(k)$
\STATE Update parameters:
$\theta^{t+1} = \theta^t - \eta^t \tilde{\bg}^t$
\STATE Update mean and standard deviation using \eqref{eq:mean-update} and \eqref{eq:std-update} respectively.
\ENDFOR
\STATE \textbf{Outputs:} $\theta^T$ and compute the overall privacy cost $(\epsilon, \delta)$ using a privacy accounting method
\end{algorithmic}
\label{alg:adaclip}
\end{algorithm}
\vspace{-1em}
In this section, we present the optimal estimator based on the noisy differentially private version of the gradients. First note that, to set the optimal values of $a^t$ and $b^t$, we need to know the mean and variance of the gradients. We propose to estimate them using noisy differentially private gradients. The full algorithm AdaCliP is presented in Algorithm~\ref{alg:adaclip}.

AdaCliP minimizes the objective function $f(\theta) = \frac1N \sum_k f_k(\theta)$ preserving privacy under $(\epsilon, \delta)$-differential privacy. 
At each iteration of SGD, AdaCliP
selects a minibatch of $B$ users. It then computes the 
stochastic gradient corresponding to each user and adds noise to the
stochastic gradient 
with optimal choices for transformation  vectors $\ba^t$ and $\bb^t$. Later AdaCliP updates the parameters (mean and variance) using noisy gradients. Notice that here the Gaussian noise is added to each individual user gradient separately instead of adding to the mean processed gradient as 
described earlier. 
Since the sum of Gaussian noises is also Gaussian noise, adding Gaussian noise to the individual user processed gradient and to the mean processed gradient is essentially equivalent.  AdaCliP also updates the mean  and variance estimates of the gradients using the noisy gradients. 

\textbf{Mean estimate}: Since there is no direct access to stochastic gradients at time $t$, $\bm^t$ is approximated by exponential average of previous noisy gradients $\tilde{\bg}^t$ (momentum style approach)
\begin{align}
\bm^t = \beta_1 \bm^{t-1} + (1-\beta_1) \tilde{\bg}^t,
\label{eq:mean-update}
\end{align}
where $\beta_1$ is a decay parameter of the exponential moving average.

\textbf{Variance estimate}:
For variance, we need to estimate $\EE(g_i^t - m_i^t)^2$ and cannot be approximated by a moving average of $\EE(\tilde{g}_i^t - m_i^t)^2$. We show that  $\EE(g_i^t - m_i^t)^2$ can be inferred from $\EE(\tilde{g}_i^t - m_i^t)^2$ as follows by assuming that clipping does not take place i.e., $\tilde g^t = g^t+b^t N^t$.
\vspace{-0.5em}
\begin{align*}
\EE (g_i^t - m_i^t)^2 = & \EE(\tilde g_i^t-m_i^t)^2 + \EE (b_i^t N_i^t)^2 + 2\EE(-b_i^t N_i^t)(g_i^t +b_i^t N_i^t - m_i^t) \\
=& \EE(\tilde g_i^t - m_i^t)^2 -\EE (b_i^t N_i^t)^2 -2\EE(b_i^t N_i^t) (g_i^t - m_i^t) \\
=& \EE(\tilde{g}_i^t - m_i^t)^2 -(b_i^t)^2\sigma^2.
\end{align*}
However, the above quantity can be quite noisy. Hence we ensure that the quantity is both upper and lower bounded as follows:
\vspace{-0.5em}
\[
(g_i^t - m_i^t)^2 \approx \min( \max(( \tilde{g}_i^t - m_i^t)^2 -(b_i^t)^2\sigma^2 , h_1), h_2),
\]
where $h_1$ and $h_2$ are small constants. We use an exponential moving average of the above quantity to estimate the variance as
\begin{align}
v_t &= \min( \max( (\tilde{g}_i^t - m_i^t)^2 -(b_i^t)^2\sigma^2, h_1), h_2) \nonumber, \\
(s_i^{t})^2 &= \beta_2 (s_i^{t-1})^2 + (1-\beta_2) v_t.
\label{eq:std-update}
\end{align}
We observed that our algorithm is \emph{robust} to parameters $\beta_1$, $\beta_2$, $h_1$, and are thus, set to $0.99$, $0.9$, $10^{-12}$ in all our experiments. We only tune $h_2$ in our experiments.

%% file: experiments.tex
\vspace{-1em}
\section{Experiments}
\vspace{-1em}
\label{sec:experiments}
\begin{figure*}[t]
\centering
\begin{minipage}{.5\textwidth}
\begin{center}
\includegraphics[height = 4.5cm]{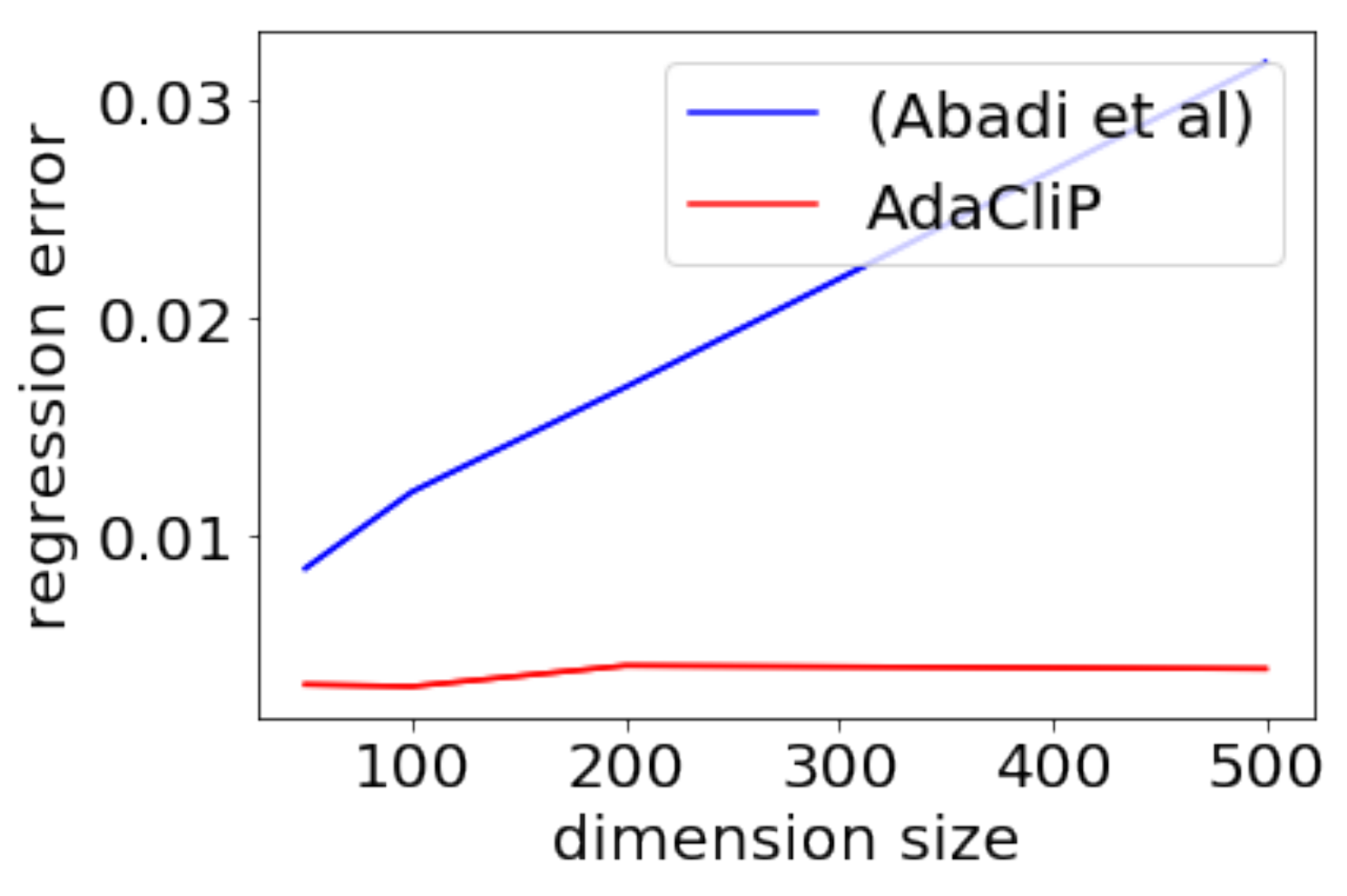}
\vspace{-0.5em}
\caption{Regression error vs dimension size}
\label{fig:regression}
\end{center}
\end{minipage}%
\begin{minipage}{.5\textwidth}
\begin{center}
\includegraphics[height=4.2cm]{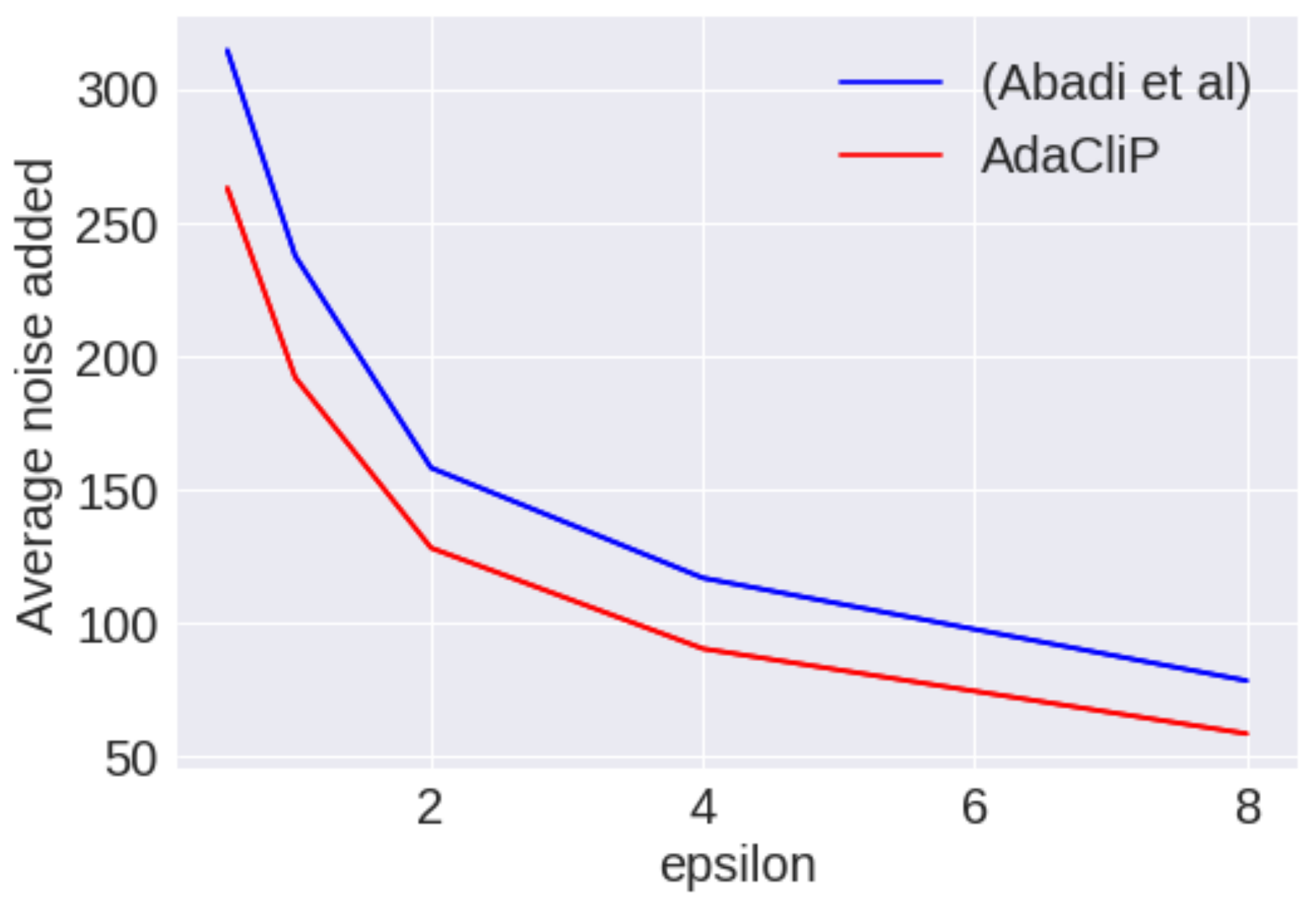}
\vspace{-0.5em}
\caption{Average noise per gradient vs $\epsilon$ for $(\epsilon,10^{-5})$-DP neural network}
\label{fig:noise_vs_epsilon}
\end{center}
\end{minipage}
\end{figure*}
\begin{table*}
    \centering
    \caption{Accuracy ($\%$) vs $\epsilon$ for $(\epsilon,10^{-5})$-DP Logistic Regression}
    \begin{tabular}{ c c c c c c c c}
    $\epsilon$     & 0.1 & 0.25 & 0.5 & 1.0 &  2.0 \\ 
    \hline
    (norm bound) & 73.45 $\pm$ 0.23 & 79.18 $\pm$ 0.18 & 84.30 $\pm$ 0.13 & 88.02 $\pm$ 0.08 & 89.65 $\pm$ 0.04\\
    (Abadi et al) & 84.23 $\pm$ 0.15 & 87.81 $\pm$ 0.10 & 90.13 $\pm$ 0.08 & 90.74 $\pm$ 0.05 & 90.97 $\pm$ 0.03 \\
    AdaCliP & \textbf{ 85.37 $\pm$ 0.19 } &  \textbf{88.11 $\pm$ 0.14 } &  \textbf{90.29 $\pm$ 0.11 }&  \textbf{90.87 $\pm$ 0.08} &  \textbf{91.15 $\pm$ 0.05} \\
    \end{tabular}
    \label{tab:log_reg}
    \caption{Accuracy ($\%$) vs $\epsilon$ for $(\epsilon,10^{-5})$-DP neural network}
    \begin{tabular}{c c c c c c}
    
    $\epsilon$     &  0.2& 0.5 & 1.0 & 2.0 & 4.0 \\
    \hline
    (Abadi et al) & 85.55 $\pm$ 0.20 & 90.23 $\pm$ 0.19 & 92.92 $\pm$ 0.18 & 94.96 $\pm$ 0.14 & 95.91 $\pm$ 0.12 \\
    AdaCliP & \textbf{ 87.18 $\pm$ 0.24 }&  \textbf{91.26 $\pm$ 0.22} &  \textbf{93.71 $\pm$ 0.20 }&  \textbf{95.56 $\pm$ 0.17} &  \textbf{96.31 $\pm$ 0.13}  \\
    \end{tabular}
    \label{tab:neural}
    
\end{table*}

We compare AdaCliP to previous methods over a synthetic example and models on MNIST and show that AdaCliP obtains as much as 1.6\% improvement in accuracy over previous methods for neural networks. We first compare 
AdaCliP with~\cite{MartinAIBIKL16} on  $\ell_2^2$-regression problem~\eqref{eq:regression}. Let $x^1, \ldots,  x^{1000}$ be such that each $x^i \in R^d$ and $x^i = (y_i,0,..,0)$ where $y_i = 1$ for $i\leq 500$ and $-1$ otherwise. We find $\theta$ that minimizes the sum of square of distances to $x^i$. 
We run both AdaCliP and 
~\cite{MartinAIBIKL16} 
using clip threshold of 1.0 and noise scale 
$\sigma = 0.1$. We run for 10 epochs with mini batch size of 1 and learning rate of $0.01$. Figure~\ref{fig:regression} shows that $\ell_2^2$ regression error of~\cite{MartinAIBIKL16} is much higher than that of AdaCliP. Furthermore, as expected AdaCliP does not add noise to dimensions other than 1 and hence its error remains independent of the number of dimensions.

We now compare AdaCliP to the previous methods
 on the MNIST dataset~\cite{YannLYP98}. MNIST consists 60,000 training images and 10,000 test images. 
 We divide each feature value by 255.0 to standardize it to $[0,1]$. We use mini batch size of 600 in all the experiments, fix $\delta = 10^{-5}$, and compare accuracy values for different $\epsilon$. We use moments account \cite{MartinAIBIKL16} to keep track of privacy loss, as it is known to give tight privacy bounds for the Gaussian mechanism. 

We first consider the logistic regression models. The $\ell_2$ sensitivity of the gradient norm can be bounded using the $\ell_2$-norm of the inputs. We compare three methods: \textit{norm bound}: adding the Gaussian noise proportional to the  sensitivity bound (the maximum $\ell_2$-norm of the inputs which is 28.0) to the gradients, \cite{MartinAIBIKL16}, and AdaCliP. For ~\cite{MartinAIBIKL16}, we clip the gradient norm at 4.0 (near the median value of the gradient norm). 
The results are in Table~\ref{tab:log_reg}. AdaCliP achieves better accuracy than both \cite{MartinAIBIKL16} and norm bound. The accuracy gains for AdaCliP over~\cite{MartinAIBIKL16} ranges from 0.2\% at $\epsilon = 2.0$ to 1.1\% at $\epsilon = 0.1$.

We then consider a neural model similar to the one in~\cite{MartinAIBIKL16}. 784 dimensional input is projected to 60 dimensions using differentially private PCA and then a neural network with a single hidden layer of 1000 units is trained on the 60 dimensional input. The privacy budget is split between PCA and neural network training. As suggested in~\cite{MartinAIBIKL16}, for~\cite{MartinAIBIKL16}, we clip the gradient norm of each layer at 4.0. 
Table~\ref{tab:neural} shows that AdaCliP consistently performs better than the previous methods. The accuracy gains for AdaCliP over~\cite{MartinAIBIKL16} ranges from 
$0.4\%$ at $\epsilon=2.0$ to $1.6\%$ at $\epsilon=0.2$.

We also compare the noise added to the gradients for AdaCliP and~\cite{MartinAIBIKL16} for the neural model by evaluating the average value of $||g^t - \tilde{g}^t||_2$, which is a combination of both clipping and additive Gaussian noise. Figure~\ref{fig:noise_vs_epsilon} shows that  AdaCliP consistently adds less noise than~\cite{MartinAIBIKL16}, which is consistent with our theory. The ratio of noises is around 0.8 for all $\epsilon$. Hence, AdaCliP achieves both better accuracy and adds smaller amount of noise compared to the Euclidean clipping of~\cite{MartinAIBIKL16}. 

%% file: conclusion.tex
\vspace{-1em}
\section{Conclusion}
\vspace{-1em}
\label{sec:conclusion}
We proposed AdaCliP, an $(\epsilon, \delta)$- differentially private SGD algorithm that adds smaller amount of noise to the gradients during training. We compared our technique with previous methods on
MNIST dataset and demonstrated that we achieve higher accuracy for the same value of $\epsilon$ and $\delta$. It would be interesting to see if instead of using a coordinate-wise gradient transform, using a matrix or low rank matrix gradient transform would give better results.

%% file: convergence.tex
\section{AdaCliP Convergence Analysis}
\label{app:convergence}
\begin{Theorem*}
Suppose the function $f$ is L-Lipschitz smooth, $\|\nabla f_k(\btheta)\| \leq G$ for all $\theta \in R^d$ and $k \in [N]$, and $\EE_k\|\nabla f_k(\btheta) - \nabla f(\theta)\|^2 \leq \sigma_g^2$ for all $\theta \in R^d$. Furthermore, suppose $\eta < \tfrac{1}{3L}$. Then, for the iterates of AdaCliP with batch size 1 and $a_t = \EE[g_t]$, we have the following:
\begin{align*}
\frac{1}{T}\sum_{t=0}^{T-1} & \EE[ \|\nabla f(\btheta^t)\|^2] \leq \frac{2[f(\btheta^0) - f(\btheta^*)]}{\eta T} + \underbrace{3L\eta\sigma_g^2}_{\substack{\text{Stochastic Gradient} \\ \text{Variance}}} 
\\
& + \underbrace{\frac{6G^2}{T} \sum_{t=0}^{T-1} \EE\left\|\frac{\bs^t}{\bb^t}\right\|^2}_{\text{Clipping bias}} 
 + \underbrace{\frac{L\eta\sigma^2}{T} \sum_{t=0}^{T-1} \|\bb^t\|^2}_{\text{Noise-addition variance}},
\end{align*}
where $s_i^t \stackrel{\Delta}{=} \sqrt{\EE (g_i^t - \EE g_i^t)^2}$.
\end{Theorem*}
\begin{proof}
Recall that the update is of the form
\begin{align*}
\btheta^{t+1} = \btheta^t - \eta \left[\frac{\bg^t - \ba^t}{\max\lbrace \|\frac{\bg^t - \ba^t}{\bb^t}, 1 \|\rbrace} + \bb^t \bN^t + \ba^t \right].
\end{align*}
For the ease of exposition, we define the following quantities:
\begin{align*}
\bc^t &= \frac{\bg^t - \ba^t}{\max\lbrace \|\frac{\bg^t - \ba^t}{\bb^t}, 1 \|\rbrace} , \\
\bDelta^t &= \bc^t - (\bg^t - \ba^t).
\end{align*}
We start with the bound following:
\begin{align*}
\EE f(\btheta^{t+1}) &\leq f(\btheta^t) + \EE \langle \nabla f(\btheta^t), \btheta^{t+1} - \btheta^t \rangle + \frac{L}{2} \EE\|\btheta_{t+1} - \btheta^t\|^2 \\
&= f(\btheta^t) - \eta \EE\langle \nabla f(\btheta^t), \bc^t + \bb^t \bN^t + \ba^t \rangle + \frac{L\eta^2}{2} \EE\left\| \bc^t + \bb^t \bN^t + \ba^t  \right\|^2 \\
&= f(\btheta^t) - \eta \EE\langle \nabla f(\btheta^t), \bc^t + \ba^t \rangle + \frac{L\eta^2}{2} \EE\left\| \bc^t + \ba^t \right\|^2 + \frac{L\eta^2}{2} \EE [\|\bb^t \bN^t\|^2] \\
&= f(\btheta^t) - \eta \EE\langle \nabla f(\btheta^t), \bc^t + \ba^t \rangle + \frac{L\eta^2}{2} \EE\left\| \bc^t + \ba^t \right\|^2 +  \frac{L\eta^2 \sigma^2}{2} \|\bb^t\|^2.
\end{align*}
The first inequality follows from Lipschitz continuous nature of the gradient. The second equality follows from the fact that $\EE[\bN^t] = 0$ and $\bN^t$ is independent of $\bg^t$ and $\ba^t$. The last equality is due to the fact that $\EE[\|N_i^t\|^2] = \sigma^2$. Note that $\bc^t + \ba^t = \bDelta^t + \bg^t$. Therefore, from the above inequality, we get
\begin{align}
\label{eq:1}
&\EE f(\btheta^{t+1})\nonumber\\
\leq& f(\btheta^t) - \eta \EE\langle \nabla f(\btheta^t), \bg^t + \bDelta^t \rangle + \frac{L\eta^2}{2} \EE\left\| \bg^t + \bDelta^t \right\|^2 +  \frac{L\eta^2 \sigma^2}{2} \|\bb^t\|^2 \nonumber \\
=& f(\btheta^t) - \eta \|\nabla f(\btheta^t)\|^2  - \eta \EE\langle \nabla f(\btheta^t), \bDelta^t \rangle + \frac{L\eta^2}{2} \EE\left\| \bg^t  - \nabla f(\btheta^t) + \nabla f(\btheta^t)+ \bDelta^t \right\|^2 +  \frac{L\eta^2 \sigma^2}{2} \|\bb^t\|^2 \nonumber \\
=& f(\btheta^t) - \eta \|\nabla f(\btheta^t)\|^2  + \eta \|\nabla f(\btheta^t)\| \EE\|\bDelta^t\| + \frac{L\eta^2}{2} \EE\left\| \bg^t  - \nabla f(\btheta^t) + \nabla f(\btheta^t)+ \bDelta^t \right\|^2 +  \frac{L\eta^2 \sigma^2}{2} \|\bb^t\|^2 \nonumber \\
\leq& f(\btheta^t) - \eta \|\nabla f(\btheta^t)\|^2  + \eta \|\nabla f(\btheta^t)\|  \EE\|\bDelta^t\| + \frac{3L\eta^2}{2}  \left[\EE\left\| \bg^t  - \nabla f(\btheta^t)\right\|^2 + \EE\|\nabla f(\btheta^t)\|^2 + \EE\|\bDelta^t \|^2 \right]\nonumber\\
&+  \frac{L\eta^2 \sigma^2}{2} \|\bb^t\|^2.
\end{align}
The last inequality follows from the fact that $\|v^1 + v^2 + v^3\|^2 \leq 3(\|v^1\|^2 + \|v^2\|^2 + \|v^3\|^2)$. Using Lemma~\ref{lem:prob-bound}, we have the following bound on $\|\nabla f(\btheta^t)\|  \EE\|\bDelta^t\| $ and $ \EE\|\bDelta^t\|^2$:
\begin{align*}
 \|\nabla f(\btheta^t)\|  \EE\|\bDelta^t\| &\leq G P(\|\bDelta^t\| > 0) \EE\left[\|\bDelta^t\| \Big\vert \|\bDelta^t\| > 0\right] \leq 2G^2 \left\|\frac{\bs^t}{\bb^t}\right\|^2 \\
  \EE\|\bDelta^t\|^2 &\leq  P(\|\bDelta^t\| > 0) \EE\left[\|\bDelta^t\|^2 \Big\vert \|\bDelta^t\| > 0\right] \leq 2G^2 \left\|\frac{\bs^t}{\bb^t}\right\|^2.
\end{align*}
The second inequality uses the fact that $\|\bDelta^t\|^2 \leq 2G^2$. Plugging in these bounds into Equation~\eqref{eq:1}, we get
\begin{align*}
\EE f(\btheta^{t+1}) &\leq f(\btheta^t) - \left(\eta -  \frac{3L\eta^2}{2}\right) \|\nabla f(\btheta^t)\|^2  + 2G^2\left(\eta + \frac{3L\eta^2}{2}\right)\left\|\frac{\bs^t}{\bb^t}\right\|^2 + \frac{3L\eta^2\sigma_g^2}{2} +  \frac{L\eta^2 \sigma^2}{2} \|\bb^t\|^2 \\
&\leq f(\btheta^t) - \frac{\eta}{2} \|\nabla f(\btheta^t)\|^2  + 2G^2\left(\eta + \frac{3L\eta^2}{2}\right)\left\|\frac{\bs^t}{\bb^t}\right\|^2 + \frac{3L\eta^2\sigma_g^2}{2} +  \frac{L\eta^2 \sigma^2}{2} \|\bb^t\|^2.
\end{align*}
Adding the above inequalities from $t=0$ to $T-1$ and by using telescoping sum, we get
\begin{align*}
\frac{1}{T}\sum_{t=0}^{T-1} \EE[ \|\nabla f(\btheta^t)\|^2] \leq \frac{2[f(\btheta^0) - f(\btheta^T)]}{\eta T} + 3L\eta\sigma_g^2 + \frac{6G^2}{T} \sum_{t=0}^{T-1} \left\|\frac{\bs^t}{\bb^t}\right\|^2 + \frac{L\eta\sigma^2}{T} \sum_{t=0}^{T-1} \|\bb^t\|^2.
\end{align*}
Here, we used the condition $\eta < \tfrac{1}{3L}$. The desired result is obtained by using the fact that $f(\btheta^T) \geq f(\btheta^*)$.
\end{proof}
\begin{Lemma}
\label{lem:prob-bound}
Let $\bs^t$ be such that $s^t_i = \EE[(\bg^{t}_i - \nabla [f(\btheta^t)]_i)^2]$. If $\ba^t = \EE [\bg^t]$ then $\bDelta^t > 0$ with at most probability $\min\{1, \|\frac{\bs^t}{\bb^t}\|^2\}$.
\end{Lemma}
\begin{proof}
We first observe that $\bDelta^t = 0$ when $\|\bg^t - \ba^t\| \leq \|\bb^t\|$. Thus, we essentially have to bound the probability that $\|\bg^t - \ba^t\| \geq \|\bb^t\|$. This follows from a simple application of Chebyshev's inequality:
\begin{align*}
P\left( \left\| \frac{\bg^t - \EE[\bg^t]}{ \bs^t \frac{\bb^t}{\bs^t}} \right\|^2 \geq 1 \right) \leq  \left\|\frac{\bs^t}{\bb^t} \right\|^2,
\end{align*}
which gives us the desired result.
\end{proof}

%% file: discussions.tex
\section{Comparison of SGD with momentum}
\label{sec:momentum}
\begin{figure*}[t]
    \centering
\subfigure[Convergence for $\epsilon=0.5$, $\delta = 10^{-5}$.]{\label{fig:conv_vs_momentum}\includegraphics[height=4.8cm]{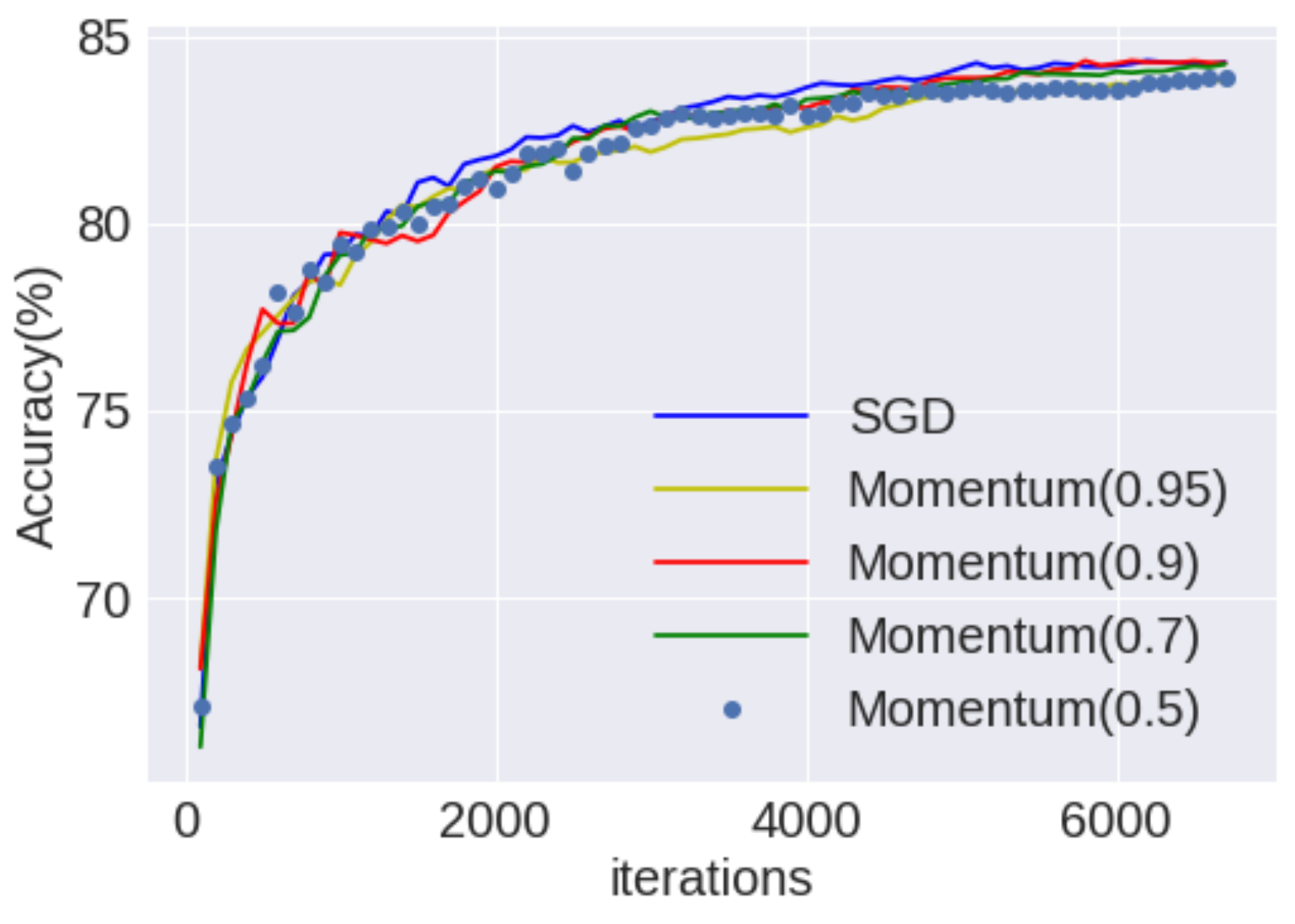}}\hspace{6em}
\subfigure[Accuracy vs $\epsilon$ for $\delta=10^{-5}$.]{\label{fig:accuracy_vs_momentum}\includegraphics[height=4.8cm]{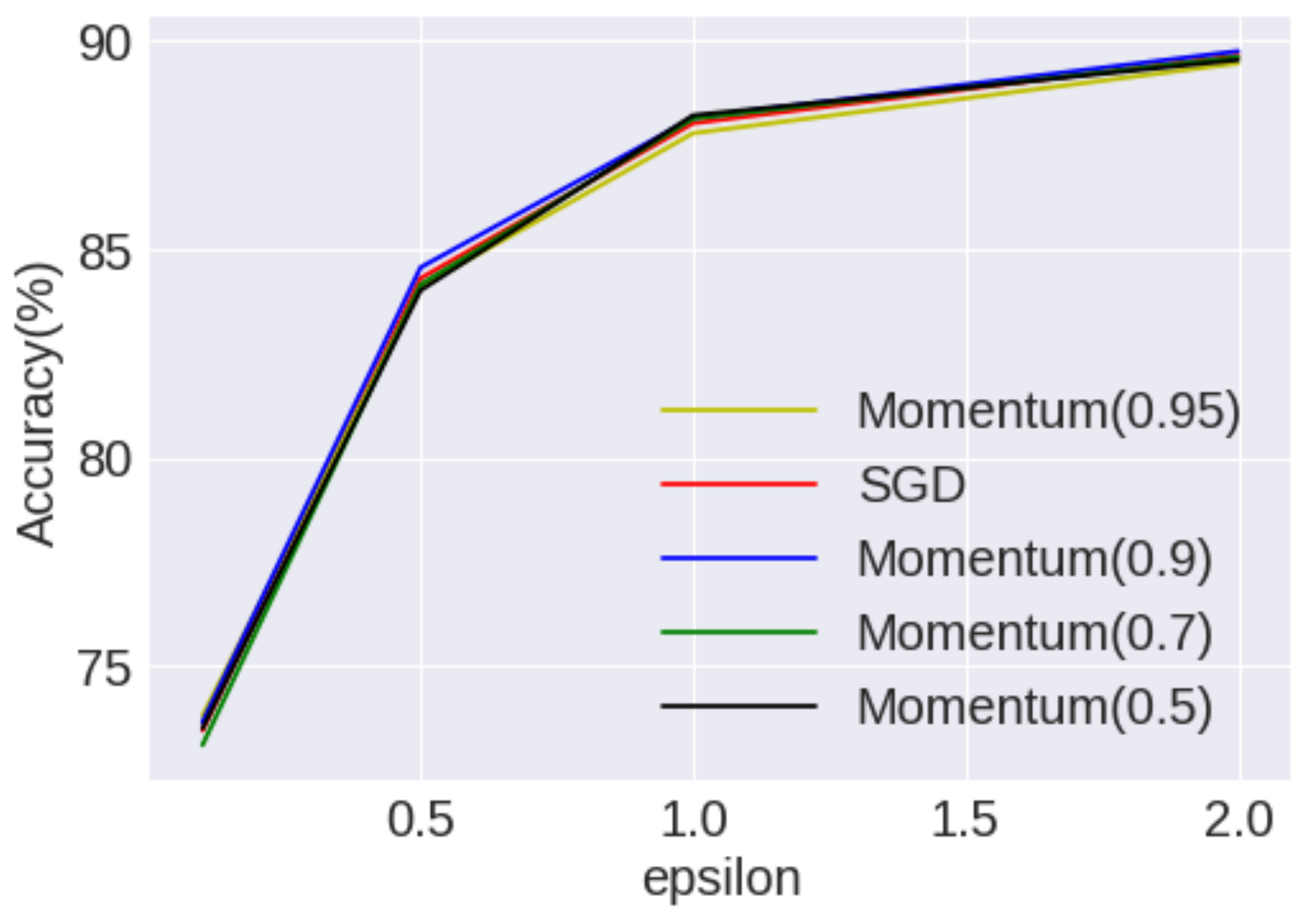}}
\vspace{-1em}
\caption{Comparison of SGD with various momentum factors}
\end{figure*}

One can ask if we can obtain benefits similar to AdaCliP by simply using momentum. We provide an intuitive reasoning why this may not be the case. Momentum maintains accumulation vector $\nu^t$ that keeps track of exponentially 
weighted averages of previous gradients. 
\setlength\abovedisplayskip{0pt}
\setlength\belowdisplayskip{0pt}
\begin{align*}
\nu^t &= \beta\nu^{t-1} + (1-\beta) \bg^{t},\\
\theta^{t} &= \theta^{t-1} - \eta \nu^t, 
\end{align*}
\setlength\belowdisplayskip{0pt}
where $\beta$ is the momentum parameter. Notice that since 
$\nu^t$ is an exponentially weighted average of previous gradients, it can also be expressed as 
\begin{align*}
    \nu^t = (1-\beta)\sum_{i=0}^t \beta^{t-i} \bg^{t}
\end{align*}
When instead of original gradients, privacy-preserving approximations 
$\tilde{\bg}^t$ are used in optimization, notice that even independent Gaussian noises added over several steps get exponentially averaged. Assuming same noise scale $\sigma$ is used over all iterations, it can be shown that 
\setlength\abovedisplayskip{0pt}
\setlength\belowdisplayskip{0pt}
\begin{align*}
    \nu^t &= (1-\beta)\sum_{i=0}^t \beta^{t-i} \tilde{\bg}^{t}\\
    &= (1-\beta)\sum_{i=0}^t \beta^{t-i} {\bg}^{t} + 
    (1-\beta)\sum_{i=0}^t \beta^{t-i} \bN^{t}\\
    &\approx (1-\beta)\sum_{i=0}^t \beta^{t-i} {\bg}^{t} + 
    \mathcal{N}(0, (1-\beta)/(1+\beta) \sigma^2 \bI). 
\end{align*}
Notice that noise added per update is factor
$\sqrt{\frac{1-\beta}{1+\beta}}$ smaller than that in SGD. This might lead one to believe that deferentially private momentum optimization might reach better model parameters compared to vanilla SGD. 

To evaluate this, consider the logistic regression task on MNIST in Section~\ref{sec:experiments}. To avoid clipping, we add noise proportional to maximum gradient norm i.e., 28. In Figure~\ref{fig:conv_vs_momentum}, we aim for $(0.5,10^{-5})$-differential privacy. Figures~\ref{fig:accuracy_vs_momentum} and~\ref{fig:conv_vs_momentum} show that SGD and momentum with various momentum factors ($\beta$) converge to almost similar accuracies.

We hypothesize that this behavior is due to the fact that under momentum, noises added across iterations are dependent. Hence, even though the noise added per iteration is small, overall noise added to sum of all gradients is the same for both SGD and momentum. For SGD, the total amount of noise added is $ \sum^T_{t=0} \bN^t$. Observe that the same holds for momentum as
\begin{align*}
\sum^T_{t=0}  \nu^{t} &= \sum^T_{t=0}  (1-\beta) \sum_{i=0}^{t} \beta^{t-i} \tilde{\bg}^t\\
&=\sum^T_{t=0}  \left [(1-\beta)\sum_{i=0}^t \beta^{t-i} {\bg}^{t} + 
   (1-\beta)\sum_{i=0}^t \beta^{t-i} \bN^{t}\right]\\
   &\approx \sum^T_{t=0}  \left[(1-\beta)\sum_{i=0}^t \beta^{t-i} {\bg}^{t} + 
   N^t\right].
\end{align*}
It would be interesting to provide better theoretical understanding for this behavior.